\documentclass{article}

\usepackage{iclr2027_conference,times}
\usepackage{amsmath,amssymb,amsthm,booktabs,bm}
\usepackage{hyperref}
\usepackage{url}
\usepackage{todonotes}
\usepackage{cleveref}
\usepackage{algorithm}
\usepackage{algpseudocode}
\usepackage{enumitem}
\usepackage{svg}
\usepackage{booktabs}
\usepackage{tabularx}
\usepackage{array}
\usepackage{subcaption}

\usepackage{tikz}
\usetikzlibrary{
    positioning,
    arrows.meta,
    calc,
    fit,
    backgrounds,
    decorations.pathreplacing
}

\crefname{theorem}{Theorem}{Theorems}
\crefname{section}{Sec.}{Secs.}
\crefname{equation}{Eq.}{Eqs.}
\crefname{figure}{Fig.}{Figs.}
\crefname{table}{Tab.}{Tabs.}
\crefname{algorithm}{Algorithm}{Algorithms}

\newtheorem{theorem}{Theorem}
\theoremstyle{remark}

\newcommand{\cG}{\mathcal{G}}

\renewcommand{\vec}[1]{\mathbf{#1}}

\newcommand{\setify}[1]{\mathcal{#1}}

\newcommand{\ex}[2][]{\mathbb{E}_{#1}\!\left[#2\right]}
\newcommand{\norm}[1]{\left\| #1\right\|}
\newcommand{\cce}[2]{C_{#1}^{\mathrm{CE}}(#2)}
\newcommand{\crel}[2]{C_{#1}^{\mathrm{rel}}(#2)}
\newcommand{\rrel}{\mathcal R_{\mathrm{rel}}}
\newcommand{\rdisc}{\mathcal R_{\mathrm{disc}}}
\newcommand{\risk}{\mathcal R}
\newcommand{\gaprel}{\Delta_{\mathrm{rel}}}
\newcommand{\gapdisc}{\Delta_{\mathrm{disc}}}
\newcommand{\gaptrain}{\Delta_{\mathrm{train}}}
\newcommand{\gapopt}{\Delta_{\mathrm{opt}}}

\newcommand{\garg}{g^{\mathrm{arg}}}

\DeclareMathOperator*{\argmin}{arg\,min}
\DeclareMathOperator*{\argmax}{arg\,max}

\title{Task-Aware Discretization of Differentiable Logic Gate Networks}

\author{Thore Gerlach \\
Independent Researcher \\
\texttt{tgerlac2@gmail.com}}

\iclrfinalcopy
\begin{document}

\maketitle

\begin{abstract}
Differentiable logic gate networks (DLGNs) enable gradient-based training of highly efficient Boolean networks by relaxing discrete logic gates during training and discretizing them for inference. Standard approaches make this discretization decision locally, typically through argmax selection and confidence- or entropy-based convergence criteria. We show that local discretization can be task-suboptimal even for globally optimal relaxed solutions, with high gate confidence providing no general guarantee, and derive bounds relating task-aware gate selection to tractable interventions in the relaxed network. Motivated by these results, we study first-order downstream task information for progressive discretization and characterize when this local approximation is reliable. Experiments on convolutional DLGNs reveal a strong locality dependence: first-order scores become unreliable when directly optimized over nonlocal interventions, but accurately assess local argmax decisions for progressive freezing.
\end{abstract}

\section{Introduction}
\label{sec:introduction}

Differentiable logic gate networks (DLGNs) enable gradient-based training
of highly efficient Boolean networks through continuous relaxations of logic gates~\citep{petersen2022deep}. At inference, the learned
network can be reduced to Boolean operations, replacing conventional
floating-point computation with simple logical operations~\cite{gerlach2025evaluation}. This makes
DLGNs particularly attractive for resource-constrained inference, and
recent work has extended them to convolutional and recurrent architectures and more
compact formulations
\citep{petersen2024convolutional,buhrer2025recurrent,gerlach2026warp,wang2026learning}.

This efficiency requires a transition from relaxed training to discrete
inference. Standard DLGNs perform this transition by independently
selecting the most probable gate at each neuron, while recent work has
focused on reducing the resulting \emph{discretization gap} through
stochastic training, forward-aligned gradient estimators, temperature
scheduling, and progressive freezing
\citep{yousefi2026mind,kim2026align,wang2026learning}.
These approaches address an important training--inference mismatch.
However, a small discretization gap does not imply that the resulting
discrete network is optimal for the downstream task.

This observation motivates a complementary question:
\emph{which discrete gates should be selected to minimize downstream
task risk?}
Prior work typically discretizes DLGNs by selecting the most probable
gate at each neuron through argmax. This local probabilistic decision
need not align with the task loss of the complete discrete network.
Indeed, we show that even globally optimal relaxed DLGNs with arbitrarily
concentrated gate distributions can retain a non-vanishing discrete
optimality gap under componentwise argmax. Thus, neither relaxed
optimality nor local gate confidence certifies task-optimal
discretization.

We address this mismatch by formulating discretization directly in terms
of downstream task risk. Conditional expectations provide an ideal
task-aware selection rule but require averaging over discrete downstream
completions. We relate this criterion to interventions in the relaxed
network and derive an efficient first-order approximation whose candidate
scores reuse gradients from ordinary backpropagation. Our analysis
characterizes the approximation error through downstream mismatch and
intervention locality. Empirically, directly optimizing the first-order
surrogate can select nonlocal interventions for which its predictions
become unreliable, whereas it accurately assesses local argmax
interventions. This motivates a progressive discretization procedure that
uses downstream task information to determine when argmax decisions can
be safely committed.
\begin{figure*}[t]
    \centering
    \resizebox{\linewidth}{!}{%
        \input{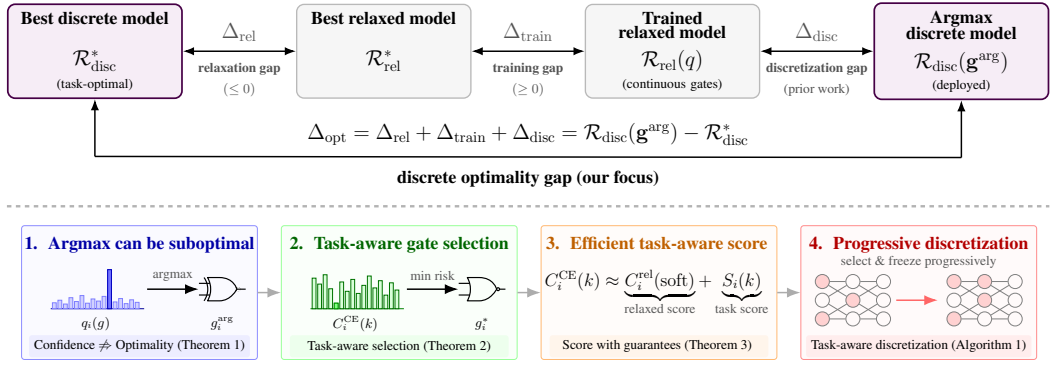}
    }
    \vspace{-0.5cm}
    \caption{Overview of the discrete optimality gap and our task-aware
        discretization framework. The top row decomposes the gap between
        the deployed argmax network and the task-optimal discrete network;
        the bottom row summarizes the theoretical motivation, task-aware
        gate selection, its efficient first-order approximation, and
        progressive discretization.
    }
    \label{fig:overview}
\end{figure*}

Our main contributions are:
\begin{itemize}[leftmargin=*, labelsep=0.5em]
    \item \textbf{A task-risk view of discretization.}
    We distinguish the conventional discretization gap from the
    \emph{discrete optimality gap}, which measures the excess task risk
    of argmax relative to the best discrete network, and relate the
    different sources of suboptimality through an exact risk
    decomposition.

    \item \textbf{Limits of argmax.}
    We prove that globally optimal DLGNs can have arbitrarily
    concentrated gate distributions while their argmax discretizations
    remain bounded away from discrete optimality.

    \item \textbf{Tractable task-aware interventions.}
    We derive task-aware gate selection from conditional risk and bound
    the error of an efficient first-order approximation in terms of
    downstream mismatch and intervention locality. The resulting scores
    reuse gradients from ordinary training.

    \item \textbf{Locality-aware empirical analysis.}
    We show that first-order scores reliably assess local argmax
    interventions, while directly optimizing them can select nonlocal
    interventions with large approximation errors. We use this insight
    for task-aware progressive freezing and compare against existing
    discretization strategies.
\end{itemize}

\section{Related Work}
\label{sec:related}

\paragraph{Differentiable logic networks.}
\citet{petersen2022deep} introduced differentiable logic gate networks
by continuously relaxing the choice among Boolean gates and selecting
the most probable gate for discrete inference.
Subsequent work extends this framework to convolutional architectures
\citep{petersen2024convolutional}, learnable connections~\citep{fojcik2026lilogic}, more efficient parameterizations
\citep{ruttgers2026light,gerlach2026warp}, and FPGA-oriented lookup-table models
\citep{buhrer2026bitlogic}.

\paragraph{Closing the discretization gap.}
Prior work aligns relaxed training and discrete inference through
stochastic gate selection with straight-through Gumbel--Softmax
\citep{kim2023deep,yousefi2026mind,gerlach2026warp}, discrete forward
passes with approximate gradients \citep{bacellar2024differentiable},
and adaptive temperature scheduling \citep{kim2026align}.
These methods modify training; we instead study discretization through
its downstream task effect.

\paragraph{Task-aware discretization.}
A closely related problem arises in differentiable neural architecture
search, where continuous operation weights are commonly converted into
a discrete architecture by selecting their largest values.
\citet{wang2021rethinking} show that these weights need not reflect an
operation's contribution to task performance and instead propose
perturbation-based, task-aware architecture selection.
Within DLGNs, \citet{wang2026learning} introduce adaptive
discretization, progressively freezing layers once their entropy
stabilizes, allowing later layers to adapt to discrete inputs.

Our approach combines task-aware evaluation with progressive
discretization: rather than freezing gates based on local confidence, we
use their predicted downstream task effect to determine when discrete
decisions can be safely committed.

\section{Background on Differentiable Logic Gate Networks}
\label{sec:background}

A DLGN is organized into layers of differentiable gate neurons. Each
neuron $i$ receives two activations $a_{i,1}$ and $a_{i,2}$ from the
preceding layer and learns which Boolean function from a candidate set
$\cG_i$ should combine them. For two-input gates, $\cG_i$ typically
contains all $16$ Boolean functions. During training, this discrete
choice is relaxed to a weighted mixture,
\[
    a_i
    =
    \sum_{g\in\cG_i} q_{i,g}\,g(a_{i,1},a_{i,2}),
    \qquad
    q_{i,g}(\tau)
    =
    \frac{\exp(\theta_{i,g}/\tau)}
         {\sum_{h\in\cG_i}\exp(\theta_{i,h}/\tau)},
\]
where $\theta_{i,g}$ are trainable logits, $g$ represents a continued relaxed surrogate of the underlying gate and the temperature $\tau>0$
controls the concentration of the gate distribution. After training,
each differentiable neuron is converted into a Boolean gate for
efficient discrete inference.

Across the $N$ neurons, the learned gate distributions induce a
factorized distribution over discrete networks, $q=\prod_{i\le N} q_i$.
Standard discretization independently selects each marginal mode,
$\garg_i=\argmax_{g\in\cG_i}q_{i,g}$, making
$\vec g^{\mathrm{arg}}$ a maximum-probability assignment under $q$,
up to ties.
Standard discretization therefore selects gates according to their
learned probabilities rather than their effect on task risk. This
distinction motivates a closer separation between existing notions of
training--inference mismatch and the quality of the resulting discrete
gate assignment.

\subsection{Discretization Gaps}
\label{sec:background-gaps}

Let $F_q$ denote the prediction of the relaxed network under gate
distributions $q$, and let $f_{\vec g}$ denote the discrete network obtained from an
assignment $\vec g=(g_1,\ldots,g_N)$. For data
$(X,Y)\sim\setify{D}$ and task loss $\ell$, its population risks are $\rrel(q)
    =
    \ex[(X,Y)\sim\setify{D}]
    {\ell(F_q(X),Y)},
$ and $
    \rdisc(\vec g)
    =
    \ex[(X,Y)\sim\setify{D}]
    {\ell(f_{\vec g}(X),Y)}$.
Discretizing a trained DLGN can degrade task performance for several
distinct reasons. Let
$
    \rrel^*
    :=
    \min_q \rrel(q),
    $ and $
    \rdisc^*
    :=
    \min_{\vec g} \rdisc(\vec g),
$
denote the optimal risks attainable by the relaxed and discrete model
classes, respectively. For a trained gate distribution $q$ and its
componentwise argmax assignment $\vec g^{\mathrm{arg}}$, we distinguish
\begin{align}
    \gaptrain:=\rrel(q)-\rrel^*,\quad
    \gaprel:=\rrel^*-\rdisc^*,\quad
    \gapdisc:=\rdisc(\vec g^{\mathrm{arg}})-\rrel(q).\nonumber
\end{align}
The training gap measures suboptimality within the relaxed model class,
the relaxation gap captures the difference between the best relaxed and
discrete models, and the discretization gap measures the change in task
risk when the trained relaxed model is converted to its argmax network.
The training gap is non-negative by definition. Moreover, since the relaxed class contains every discrete network, $\gaptrain\ge 0$ and $\gaprel\le 0$, while $\gapdisc$ is signed.

These gaps are distinct from the question of whether argmax produces a
task-optimal discrete network. We define the \emph{discrete optimality
gap} as
\begin{equation}
    \gapopt
    :=\gaprel
        +
        \gaptrain
        +
        \gapdisc =
    \rdisc(\vec \garg)-\rdisc^*
    \geq 0.
    \nonumber
\end{equation}
Thus, a suboptimal deployed network may result from imperfect relaxed
training, from the relaxation itself, or from converting the learned
relaxed model into a discrete gate assignment.

\citet{kim2026align} further decomposes the conventional discretization gap into a \emph{selection gap} from replacing soft gate mixtures by argmax gates and a \emph{computation gap} from subsequently evaluating them as Boolean functions. The computation gap vanishes for binary inputs, where the continuous gate extensions coincide with their truth tables. Our focus is complementary: even if this training--inference gap vanishes, argmax need not minimize risk among discrete networks.

\section{A Framework for Task-Aware Discretization}
\label{sec:framework}

Prior work has related accurate argmax discretization to the concentration
of learned gate distributions: DLGN neurons are observed to converge toward
individual gates, while diffuse gate distributions have been associated with
discretization errors~\cite{petersen2022deep,petersen2024convolutional,yousefi2026mind,ruttgers2026light,wang2026learning}.
This suggests that argmax should become reliable once the learned
distributions are sufficiently concentrated.
We show that local concentration alone cannot provide such a guarantee,
even when the relaxed network is globally optimal.

Let $\delta_g$ denote the distribution assigning probability one to
gate $g\in\cG$. We call a continuous function
$D$ on gate distributions a \emph{local dispersion measure} if
$D(\delta_g)=0$ for every $g\in\cG$.
This includes, for example, Shannon entropy, Gini impurity, and the
non-modal probability $1-\max_g q(g)$.
\begin{theorem}[Local concentration does not guarantee task-optimal argmax]
\label{thm:concentration-insufficient}
Let $\ell$ be a loss and let $\setify{Z}$ denote the prediction space of
the relaxed network.
Suppose there exist a data distribution $\setify{D}$ and predictions
$z_0,z^*\in\setify{Z}$ such that
\begin{equation}
    z^*
    \in
    \argmin_{z\in\setify{Z}}
    \ex[(X,Y)\sim\setify{D}]
    {\ell(z,Y)}\quad \text{and}\quad 
    \ex[(X,Y)\sim\setify{D}]
    {\ell(z_0,Y)-\ell(z^*,Y)}
    = c > 0.
    \label{eq:risk-separation}
\end{equation}
Then there exists a sequence of DLGNs with $N\to\infty$ trainable
neurons and globally optimal relaxed gate distributions $q^*$ such that,
\begin{equation}
    \gapopt
    =c,
    \quad \text{and}\quad 
    \max_{1\le i\le N}D(q_{i}^*)
    \longrightarrow 0.
    \label{eq:local-concentration}
\end{equation}
\end{theorem}
The proof is found in~\cref{app:anti-argmax-proof}.
This theorem highlights a limitation of purely local confidence measures, that is local concentration of the
gate distributions does not, by itself, control their downstream task-level
effect. Concrete instances of~\cref{thm:concentration-insufficient} for
cross-entropy and squared loss are given in~\cref{app:counterexamples}.

\subsection{Optimal Task-Aware Discretization via Conditional Expectations}
\label{sec:conditional-expectation}

The previous result shows that local concentration does not determine which
discrete gate assignment minimizes task risk. We therefore consider a
discretization rule that evaluates gate choices directly through their
downstream loss.

Suppose that gates $g_1,\ldots,g_{i-1}$ have already been fixed. For a
candidate gate $k\in\cG$ at neuron $i$, define its conditional task risk as
\begin{equation}
    \cce{i}{k}
    :=
    \ex[\vec{G}_{>i}\sim q_{>i}]
    {
        \rdisc(\vec g_{<i},k,\vec{G}_{>i})
    },
    \label{eq:conditional-task-risk}
\end{equation}
where $\vec g_{<i}=(g_1,\ldots,g_{i-1})$ and
$q_{>i}=\prod_{j>i}q_j$.
Selecting
$
    g_i^{\mathrm{CE}}
    \in
    \argmin_{k\in\cG}
    \cce{i}{k}
$
sequentially is the method of conditional expectations~\citep{motwani1996randomized} and yields a fully discrete assignment
$\vec{g}^{\mathrm{CE}}$. 
Since each decision minimizes the
conditional expected risk, $\vec{g}^{\mathrm{CE}}$ satisfies
\begin{equation}
    \rdisc^*
    \leq
    \rdisc(\vec{g}^{\mathrm{CE}})
    \leq
    \ex[\vec{G}\sim q]
     {\rdisc(\vec{G})}.
    \label{eq:ce-risk-bound}
\end{equation}
A proof is given in~\cref{app:ce-proof}.
Evaluating~\cref{eq:conditional-task-risk} exactly requires
averaging over the exponentially many discrete completions of the remaining
network. We therefore seek a quantity that can be evaluated directly through
the trained relaxed DLGN.

For a fixed prefix $\vec g_{<i}$ and candidate gate $k$, let
$\mu_{i,k}(X)
:=\ex[\vec G_{>i}\sim q_{>i}]
{f_{\vec g_{<i},k,\vec G_{>i}}(X)}$
denote the mean prediction and $\sigma^2_{i,k}(X):=\mathbb E_{\vec G_{>i}\sim q_{>i}}[
            \norm{
                f_{\vec g_{<i},k,\vec G_{>i}}(X)
                -
                \mu_{i,k}(X)
            }^2
    ],$ the variance over discrete downstream completions.
The corresponding relaxed intervention is
$F_{i,k}(X):=F_{\vec g_{<i},k,q_{>i}}(X)$, with risk
$\crel{i}{k}:=\ex[(X,Y)\sim\setify{D}]{
\ell(F_{i,k}(X),Y)}$.
In general, $F_{i,k}\neq\mu_{i,k}$ because relaxed propagation need not
equal averaging over discrete completions. The following result quantifies
the resulting approximation error.

\begin{theorem}[Approximation of conditional task risk]
\label{thm:ce-relaxed}
Assume that $\ell$ is $L$-Lipschitz and define
$V_{i,k}:=\ex[(X,Y)\sim\setify{D}]{\sigma_{i,k}(X)}$ and
$P_{i,k}:=\ex[(X,Y)\sim\setify{D}]{\norm{\mu_{i,k}(X)-F_{i,k}(X)}}$.
Then
\begin{equation}
    \left|
        \cce{i}{k}
        -
        \crel{i}{k}
    \right|
    \leq
    L\left(V_{i,k}+P_{i,k}\right).
    \label{eq:ce-relaxed-bound}
\end{equation}
\end{theorem}
The proof is given in~\cref{app:ce-proof}.
The two terms in~\cref{eq:ce-relaxed-bound} identify distinct sources of
approximation error. $V_{i,k}$ measures variability among discrete downstream
completions, whereas $P_{i,k}$ measures the mismatch between their mean
prediction and the deterministic relaxed forward pass. Hence relaxed
interventions approximate the ideal conditional task criterion whenever
both downstream variability and relaxation-induced propagation mismatch
are small.

\subsection{First-Order Task-Aware Gate Selection}
\label{sec:first-order-selection}

Although relaxed interventions avoid averaging over discrete downstream
completions, evaluating $\crel{i}{k}$ for every candidate still requires
many forward passes. We therefore estimate the task effect of selecting
gate $k$ from local first-order information. Let $a_i(\cdot)$ be the relaxed
output of neuron $i$ and $a_{i,k}(\cdot)$ its output under gate $k$. We define
the first-order score
\begin{equation}
    S_i(k)
    :=
    \ex[(X,Y)\sim\setify{D}]
    {
        \frac{\partial \ell}{\partial a_i}
        \left(a_{i,k}(X)-a_i(X)\right)
    }.
    \nonumber
\end{equation}
Negative scores predict a decrease in task risk, while positive scores
predict an increase.
\begin{theorem}[First-order task-aware approximation]
\label{thm:first-order-selection}
Assume the conditions of~\cref{thm:ce-relaxed} and that the sample-wise
downstream loss is twice differentiable with respect to the scalar
activation $a_i$ and upper bounded by $M_i>0$. For each $k\in\cG$, let
$
    \epsilon_{i,k}
    :=
    L_\ell(V_{i,k}+P_{i,k})
    +
    \frac{M_i}{2}
    \ex[(X,Y)\sim\setify{D}]
    {
        \Delta a_{i,k}(X)^2
    }
$
and
$
    \epsilon_i:=\max_{k\in\cG}\epsilon_{i,k}.
$
With the first-order choice
$k_i^{\mathrm{FO}}\in\argmin_{k\in\cG}S_i(k)$, we then obtain
\begin{equation}
    \left|
        \cce{i}{k}
        -
        \left[
            \crel{i}{\mathrm{soft}}+S_i(k)
        \right]
    \right|
    \leq
    \epsilon_{i,k},
    \quad \text{and thus} \quad
    \cce{i}{k_i^{\mathrm{FO}}}
    -
    \min_{k\in\cG}\cce{i}{k}
    \leq
    2\epsilon_i.
    \label{eq:first-order-ce-bound}
\end{equation}
Moreover, relative to
$k_i^{\mathrm{arg}}\in\argmax_{k\in\cG}q_i(k)$, let
$\Gamma_i:=S_i(k_i^{\mathrm{arg}})-S_i(k_i^{\mathrm{FO}})$.
Then
\begin{equation}
    \Gamma_i>2\epsilon_i
    \quad\Longrightarrow\quad
    \cce{i}{k_i^{\mathrm{FO}}}
    <
    \cce{i}{k_i^{\mathrm{arg}}}.
    \label{eq:fo-beats-argmax}
\end{equation}
\end{theorem}
The proof can be found in~\cref{app:first-order-proof}.
The approximation error is separated into two
sources. The term $L(V_{i,k}+P_{i,k})$ measures the error incurred by
replacing discrete downstream completions with the relaxed network, while
the curvature term controls the error of the local first-order
approximation. 

\subsection{Task-Aware Progressive Freezing}
\label{sec:progressive-discretization}

The population score $S_i(k)$ is not directly available during training.
We estimate it on the current minibatch $\setify{B}_t$ by averaging
$s_{i,k}=(\partial\ell/\partial a_i)(a_{i,k}-a_i)$, yielding
$\widehat{S}_{i,t}(k)=|\setify{B}_t|^{-1}
\sum_{b\in\setify{B}_t}s_{i,k}^{(b)}$.
The downstream sensitivities $\partial\ell/\partial a_i$ are already
available from backpropagation, while candidate gate outputs can be
evaluated locally without additional downstream forward passes.

For progressive discretization, we use this task information
conservatively to assess the gate already preferred by the learned
distribution. Let
$k_i^{\mathrm{arg}}=\argmax_{k\in\cG_i}q_i(k)$.
Because minibatch scores are noisy, we maintain an exponential moving
average (EMA) of its predicted freeze shock,
\[
    \mu_i^{(t)}
    =
    \beta\mu_i^{(t-1)}
    +(1-\beta)
    \widehat S_{i,t}(k_i^{\mathrm{arg}}).
\]
A neuron is frozen to $k_i^{\mathrm{arg}}$ once this score remains
negative for $P$ consecutive evaluations. Thus, unlike confidence- or
entropy-based criteria, freezing is triggered by the predicted downstream
task effect of committing the argmax gate. The procedure is summarized
in~\cref{alg:progressive-discretization}.

\begin{algorithm}[t]
\caption{Task-Aware Progressive Freezing}
\label{alg:progressive-discretization}
\begin{algorithmic}[1]
\Require scoring interval $K$, EMA factor $\beta$, patience $P$
\State Train the relaxed network for an initial burn-in period
\For{each subsequent training step $t$}
    \State Perform a standard training step
    \If{$t \bmod K = 0$}
        \For{each unfrozen decision $i$}
            \State $k_i^{\mathrm{arg}}
                \gets \argmax_{k\in\cG_i}q_i(k)$
            \State Compute $\widehat S_{i,t}(k_i^{\mathrm{arg}})$
                   and update $\mu_i$
            \If{$\mu_i<0$ for $P$ evaluations}
                \State Freeze $i$ to $k_i^{\mathrm{arg}}$
            \EndIf
        \EndFor
    \EndIf
\EndFor
\end{algorithmic}
\end{algorithm}

\section{Experiments}
\label{sec:experiments}

Our experiments address four questions. First, can local statistics
predict the task effect of discretizing individual neurons? Second, why
does directly optimizing the first-order task-aware surrogate become
unreliable? Third, can first-order task information instead provide a
robust criterion for progressive argmax freezing? Finally, how does this
procedure compare with existing discretization strategies? We study
these questions on convolutional DLGNs, beginning with the local
intervention analysis in \cref{fig:correlation-hexbin}.

\paragraph{Experimental setup.}
We evaluate on CIFAR-10, using 45k training and 5k validation examples,
cross-entropy loss, and GroupSum aggregation to obtain class logits.
We study two convolutional DLGN configurations chosen to closely follow
established architectures and training procedures. \emph{Compact-S}
follows the method described in~\citep{wang2026learning}, using uniform thresholding, Adam
optimization, and connection learning through resampling; models are
trained for 600k steps. \emph{TreeLogicNet-S} follows the TreeLogicNet
architecture~\cite{petersen2024convolutional}
, using distributive thresholding, AdamW with weight decay
$0.002$, and 200k training steps. Unless stated otherwise, we evaluate
five independent seeds per architecture. The two settings test whether
our observations persist across distinct DLGN architectures and training
procedures. Further architectural and training details are provided in
Appendix~\ref{app:experimental-details}.

\subsection{Predicting Discretization Effects}
\label{sec:exp-prediction}

\begin{figure*}[t]
    \centering
    \includegraphics[width=1.0\textwidth]{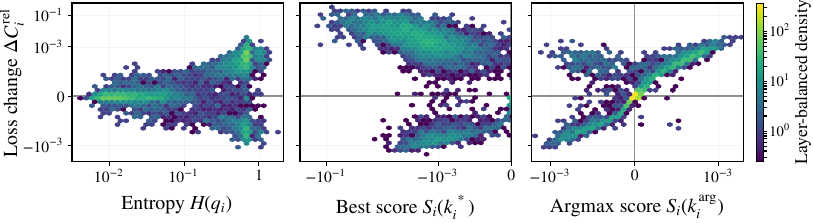}
    \vspace{-0.5cm}
    \caption{\textbf{Predicting freeze effects.}
    Measured loss change from freezing to the argmax gate versus gate
    entropy (left), from freezing to the score-selected gate
    $k_i^{*}=\argmin_k S_i(k)$ versus its first-order score
    (center), and from freezing to the argmax gate versus its first-order
    score (right). Score selection is poorly aligned with the
    realized effect, whereas the argmax freeze shock is strongly
    predictive.}
    \label{fig:correlation-hexbin}
\end{figure*}

\begin{table*}[t]
    \centering
    \caption{\textbf{Layer-wise Spearman correlations} between local
    discretization criteria and realized freeze effects. Values are mean
    $\pm$ sample standard deviation across five seeds, after averaging
    checkpoints within each seed (four for Compact-S and five for
    TreeLogicNet-S).}
    \label{tab:layerwise-correlations}
    \vspace{0.2cm}
    \scriptsize
    \setlength{\tabcolsep}{4pt}
    \begin{tabularx}{\textwidth}{
        l
        *{3}{>{\centering\arraybackslash}X}
        *{3}{>{\centering\arraybackslash}X}
    }
        \toprule
        & \multicolumn{3}{c}{\textbf{Compact-S}}
        & \multicolumn{3}{c}{\textbf{TreeLogicNet-S}} \\
        \cmidrule(lr){2-4}
        \cmidrule(lr){5-7}
        \textbf{Layer}
        & \textbf{Entropy}
        & \textbf{Score}
        & \textbf{Argmax}
        & \textbf{Entropy}
        & \textbf{Score}
        & \textbf{Argmax} \\
        \midrule
        Conv 1
        & $+0.301 \pm 0.045$
        & $-0.576 \pm 0.038$
        & $+0.742 \pm 0.047$
        & $+0.450 \pm 0.082$
        & $-0.356 \pm 0.106$
        & $+0.583 \pm 0.081$ \\

        Conv 2
        & $+0.370 \pm 0.053$
        & $-0.625 \pm 0.021$
        & $+0.792 \pm 0.053$
        & $+0.257 \pm 0.063$
        & $-0.291 \pm 0.056$
        & $+0.766 \pm 0.039$ \\

        Conv 3
        & $+0.281 \pm 0.018$
        & $-0.304 \pm 0.074$
        & $+0.868 \pm 0.021$
        & $+0.076 \pm 0.016$
        & $+0.042 \pm 0.049$
        & $+0.762 \pm 0.034$ \\

        Conv 4
        & $+0.048 \pm 0.045$
        & $-0.242 \pm 0.081$
        & $+0.991 \pm 0.004$
        & $+0.024 \pm 0.019$
        & $+0.454 \pm 0.045$
        & $+0.676 \pm 0.013$ \\
        \bottomrule
    \end{tabularx}
\end{table*}

\Cref{fig:correlation-hexbin} compares three local quantities with the
realized task-loss change from fixing a single gate while keeping the
remaining network relaxed, defined as
$\Delta C_i^{\mathrm{rel}}(k)
:=C_i^{\mathrm{rel}}(k)-C_i^{\mathrm{rel}}(\mathrm{soft})$.
For entropy and the argmax freeze shock, we measure this change after
freezing to $k_i^{\mathrm{arg}}=\argmax_k q_i(k)$; for task-aware
selection, we instead freeze to
$k_i^{*}=\argmin_k S_i(k)$.
Gate entropy shows little relation to the resulting argmax freeze effect.
Surprisingly, selecting the gate with the lowest first-order score does
not improve this correspondence: predicted beneficial interventions can
still substantially increase the loss. In contrast, the first-order
score of the argmax gate is strongly aligned with the realized effect,
including its sign.

We quantify these relationships using Spearman rank correlation.
As summarized in~\cref{tab:layerwise-correlations}, the qualitative
pattern persists across layers, checkpoints, seeds, and both
architectures. Argmax selections are consistently strongly positively
correlated with the realized freeze effect, whereas entropy is generally
weaker and score selections are substantially less consistent, including
strong negative correlations in several layers. Thus, first-order task
information is useful for \emph{assessing} the argmax intervention but
becomes unreliable when directly optimized over candidate gates. We
investigate this discrepancy in \cref{fig:locality}, focusing on the
locality of the corresponding interventions.

\subsection{Locality of First-Order Interventions}
\label{sec:exp-locality}

\label{sec:exp-freezing}

\begin{figure*}[t]
    \centering
    \includegraphics[width=1.0\textwidth]{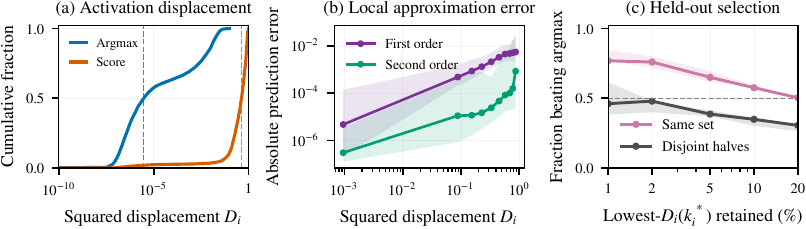}
    \vspace{-0.5cm}
    \caption{\textbf{Locality of first-order interventions.}
(a) Score selection produces substantially larger activation
displacements than argmax, with median squared displacements of
$0.42$ and $2.9\times10^{-6}$, respectively.
(b) Approximation error increases with displacement for both first- and
second-order scores, with median within-neuron Spearman correlations of
$0.88$ and $0.80$, respectively; second order substantially reduces the
typical error.
(c) The apparent advantage of selecting low-displacement task-aware
interventions disappears when gate selection and evaluation use
disjoint validation data.}
    \label{fig:locality}
\end{figure*}

We next investigate why the first-order score is predictive for argmax
interventions but unreliable when directly optimized over candidate
gates. \Cref{fig:locality}a first examines how far a discrete intervention moves
a neuron from the relaxed operating point. We quantify this by its
squared activation displacement
$D_i(k):=\mathbb E_{(X,Y)\sim\setify{D}}[\Delta a_{i,k}^2(X)]$, where $\Delta a_{i,k}(\cdot)=a_{i,k}(\cdot)-a_i(\cdot)$, $a_i(\cdot)$ is the relaxed
output of neuron $i$ and $a_{i,k}(\cdot)$ its output under gate $k$. 
Argmax and score selection operate in strikingly different regimes:
their median displacements are $2.9\times10^{-6}$ and $0.42$,
respectively. Thus, while argmax typically constitutes a highly local
intervention, directly minimizing $S_i(k)$ can select gates far from the
point at which the first-order approximation is evaluated.

\Cref{fig:locality}b relates displacement to the approximation error.
For first order, we measure the prediction error
$|\Delta C_i^{\mathrm{rel}}(k)-S_i(k)|$; for second order, we
add the directional curvature correction, that is $|\Delta C_i^{\mathrm{rel}}(k)-S_i(k)-\frac{1}{2}\Delta a_{i,k}^{\top}H_i\Delta a_{i,k}|$.
Across the 16 candidate gates of each neuron, approximation error grows
strongly with $D_i(k)$, with median within-neuron Spearman correlations
of $0.88$ and $0.80$ for first and second order, respectively.
Second order substantially reduces the typical error, confirming that
curvature explains an important part of the first-order residual, but
the remaining displacement dependence shows that nonlocal interventions
remain difficult to approximate. This is consistent with the curvature
term in \cref{thm:first-order-selection}, whose bound grows with
$D_i(k)$.

Finally, \cref{fig:locality}c asks whether restricting score
selection to local interventions can recover a reliable advantage over
argmax. Among neurons where the two rules disagree, we retain increasing
fractions with the smallest $D_i(k_i^*)$. When the same validation data
are used for selection and evaluation, the most local score-selected choices
appear to outperform argmax frequently. This advantage disappears when
gates are selected on one validation half and evaluated on the other.
Thus, although locality governs approximation quality, restricting
selection to local interventions does not yield a robust task-aware
replacement rule. We therefore use first-order information
conservatively to determine when argmax decisions can be committed.

\subsection{Progressive Freezing Dynamics}
\label{sec:exp-freezing}

\begin{figure*}[t]
    \centering
    \begin{subfigure}{0.45\textwidth}
        \centering
        \includegraphics[width=1.0\textwidth]{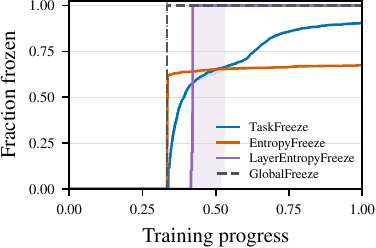}
        \caption{Global freezing dynamics.}
        \label{fig:freezing-global}
    \end{subfigure}
    \hfill
    \begin{subfigure}{0.45\textwidth}
        \centering
        \includegraphics[width=1.0\textwidth]{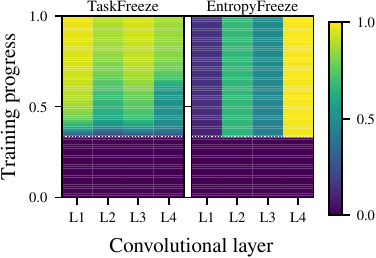}
        \caption{Layer-wise freezing dynamics.}
        \label{fig:freezing-layerwise}
    \end{subfigure}

    \caption{\textbf{Progressive freezing dynamics on Compact-S.}
    (a) Fraction of frozen neurons throughout training for different
    discretization criteria. (b) Layer-wise frozen fractions show that
    TaskFreeze progressively commits individual decisions, whereas
    EntropyFreeze induces a markedly different layer-dependent
    schedule.}
    \label{fig:freezing-dynamics}
\end{figure*}

We compare our task-aware neuron-wise criterion, \emph{TaskFreeze}, with
three freezing baselines. \emph{EntropyFreeze} freezes individual
neurons once their gate entropy has converged; \emph{GlobalFreeze} freezes all neurons simultaneously at the
discretization start and \emph{LayerEntropyFreeze}
applies the layer-wise entropy criterion of \citet{wang2026learning}. We additionally use \emph{FinalArgmax} for the
standard setting in which training remains relaxed and argmax
discretization is applied only after training.

\Cref{fig:freezing-dynamics}a shows the fraction of frozen
neurons on Compact-S. GlobalFreeze commits the complete network at the
discretization start, while LayerEntropyFreeze similarly produces an abrupt
transition once its layer-wise criterion is met. EntropyFreeze freezes
a large fraction of neurons early but subsequently changes little.
TaskFreeze instead produces a gradual trajectory: individual argmax
decisions become eligible throughout the remaining training period,
resulting in a substantially smoother transition toward the discrete
network.

The layer-wise view in \cref{fig:freezing-dynamics}b further highlights
the difference between the two neuron-wise criteria. EntropyFreeze
induces a strongly layer-dependent pattern, freezing most neurons in
deeper layers early while leaving substantially fewer decisions frozen
in the first layer. TaskFreeze distributes freezing more progressively
across depth and time, with the final layer remaining more conservative.
Thus, task-level eligibility does not simply reproduce local entropy
convergence, but induces a distinct discretization schedule based on the
predicted downstream effect of each argmax decision.
\begin{table*}[t]
    \centering
    \caption{\textbf{Final discrete CIFAR-10 test accuracy (\%).}
    Results are reported at the 400k discretization start for Compact
    models and the 25k start for TreeLogicNet models. For Compact-S, we
    compare continued and stopped connection resampling after
    discretization begins. Values are mean $\pm$ sample standard
    deviation across available seeds.}
    \label{tab:final-accuracy}
    \scriptsize
    \setlength{\tabcolsep}{1pt}

    \begin{tabularx}{\textwidth}{
        >{\raggedright\arraybackslash}X
        *{5}{>{\centering\arraybackslash}X}
    }
        \toprule
        & \multicolumn{2}{c}{\textbf{Compact-S}}
        & \textbf{Compact-M}
        & \textbf{TreeLogicNet-S}
        & \textbf{TreeLogicNet-M} \\
        \cmidrule(lr){2-3}
        \textbf{Method}
        & \textbf{Resampling}
        & \textbf{Stopped}
        & \textbf{Resampling}
        &
        & \\
        \midrule

        TaskFreeze (ours)
        & $66.25 \pm 0.27$
        & $65.93 \pm 0.15$
        & $69.34 \pm 0.46$
        & $60.76 \pm 0.46$
        & $71.31 \pm 0.04$ \\

        EntropyFreeze (ours)
        & $64.11 \pm 0.53$
        & $63.88 \pm 0.76$
        & $68.47 \pm 0.54$
        & $59.97 \pm 0.84$
        & $70.95 \pm 0.23$ \\

        GlobalFreeze (ours)
        & $66.41 \pm 0.21$
        & $66.12 \pm 0.37$
        & $69.26 \pm 0.21$
        & $60.50 \pm 0.52$
        & $71.15 \pm 0.04$ \\
        \midrule
        LayerEntropyFreeze
        & $64.10 \pm 0.45$
        & $65.99 \pm 0.17$
        & $67.67 \pm 0.93$
        & $60.63 \pm 0.55$
        & $71.44 \pm 0.01$ \\

        FinalArgmax
        & $64.63 \pm 0.61$
        & $64.63 \pm 0.61$
        & $67.67 \pm 0.39$
        & $59.29 \pm 0.78$
        & $70.49 \pm 0.27$ \\

        CAGE
        & $63.32 \pm 0.63$
        & $64.11 \pm 0.39$
        & $68.23 \pm 0.12$
        & $59.98 \pm 0.35$
        & $71.05 \pm 0.23$ \\

        \bottomrule
    \end{tabularx}
\end{table*}
\subsection{Final Discrete Performance}
\label{sec:exp-performance}

Finally, \cref{tab:final-accuracy} compares the resulting discrete
networks on the CIFAR-10 test set. TaskFreeze is consistently competitive
across both architecture families and model sizes, and improves over
EntropyFreeze in all reported settings. Interestingly, the simple
GlobalFreeze baseline is similarly strong, slightly outperforming
TaskFreeze on Compact-S while remaining comparable on the other models.
This suggests that exposing the network to discrete computation during
training is itself an important factor, while carefully scheduling
individual freeze decisions does not consistently improve final
accuracy.

Both neuron-wise criteria are compared with the layer-wise entropy
strategy, fully relaxed training with final argmax discretization, and
CAGE. Overall, early or progressive argmax freezing generally improves
over applying argmax only after relaxed training, while no single
freezing schedule consistently dominates across architectures. Detailed
results across discretization start points and connection-resampling
settings are provided in~\cref{app:discretization-start}.

\section{Discussion}
\label{sec:discussion}

Our results reveal an important distinction between using task
information to \emph{select} and to \emph{assess} a discrete
intervention. Although the first-order score approximates the downstream
effect of fixing a gate, directly minimizing this surrogate preferentially
selects large interventions for which the local approximation is least
reliable. Second-order corrections substantially reduce the typical
prediction error, confirming an important role of curvature, but remain
strongly dependent on intervention magnitude. In contrast, the argmax
gate learned by the relaxed network is typically close to the relaxed
operating point, making its first-order freeze shock a reliable local
diagnostic.
These findings persist across two architecture families, multiple layers
and training checkpoints, five-seed S-model experiments, and different
discretization starts. For Compact-S, they also persist with both
continued and stopped connection resampling.

Motivated by this distinction, TaskFreeze leaves gate selection to the
learned distribution and uses task information only to determine when a
decision is favorable to commit. Its gradual freezing dynamics differ
markedly from entropy-based criteria, yet its final accuracy is often
comparable to the much simpler GlobalFreeze baseline. This suggests that
adaptation to discrete computation during continued training is itself
an important source of performance, while fine-grained scheduling does
not universally improve the final discrete model. Developing reliable
task-aware gate replacement beyond the local regime therefore remains
an open problem, potentially requiring nonlocal or jointly optimized
discretization criteria.

\section{Conclusion}

We studied DLGN discretization from the perspective of downstream task
risk. We showed that even globally optimal relaxed networks with highly
concentrated gate distributions can remain suboptimal under
componentwise argmax, and derived conditional-risk and first-order
criteria for task-aware gate decisions. Empirically, directly optimizing
the first-order surrogate can leave its local regime, whereas the same
signal accurately assesses the naturally local argmax intervention.
TaskFreeze exploits this distinction for progressive discretization and
is competitive across two convolutional DLGN families. More broadly, our
results identify intervention locality as a central consideration for
task-aware discretization and motivate methods that can account for
nonlocal interactions between discrete decisions.

\section*{AI Use Statement}
Generative AI tools were used to assist with drafting and editing
manuscript text, developing and checking theoretical arguments and
counterexamples, designing diagnostic experiments, interpreting
experimental results, organizing and reviewing literature, and checking
LaTeX structure. AI-assisted mathematical arguments, experimental
analyses, citations, and code were independently reviewed and verified
by the authors. The authors take responsibility for the final content of the work, including all text, claims, results, and artifacts produced
with the assistance of generative AI.

\section*{Reproducibility statement}
Complete proofs are provided in~\cref{app:proofs}. \cref{app:experimental-details} specifies the model architectures, optimization and freezing hyperparameters, diagnostic protocol, and additional experiments. Code required to reproduce the reported experiments is included in the anonymized supplementary material.

\bibliography{references}
\bibliographystyle{iclr2027_conference}
\appendix 

\section{Additional Theory and Proofs}
\label{app:proofs}
\subsection{Proof of~\cref{thm:concentration-insufficient}}
\label{app:anti-argmax-proof}
\begin{proof}
Fix an integer $m\geq 1$ and consider any $N>2m$.
Construct a DLGN with $N$ trainable neurons whose candidate
gates include the constant gates $g_0\equiv 0$ and $g_1\equiv 1$.
Let its output be
\begin{equation}
   f_{\vec{g}}(X)
=
z_0
+
\frac{z^*-z_0}{m}
\sum_{i\le N} g_i.
    \nonumber
\end{equation}
For every trainable neuron $i$, define the relaxed gate distribution
\begin{equation}
    q_{i}^*(g_1)=\frac{m}{N},
    \qquad
    q_{i}^*(g_0)=1-\frac{m}{N},
    \nonumber
\end{equation}
with zero probability assigned to all remaining gates.
Under the relaxed forward pass, each neuron has output $m/N$.
Consequently,
\begin{equation}
    F_{q^*}(X)
    =
    z_0
    +
    \frac{z^*-z_0}{m}
    \sum_{i\le N}\frac{m}{N}
    =
    z^*. \nonumber
\end{equation}
By assumption, $z^*$ globally minimizes the population risk over the
prediction space attainable by the relaxation. Hence $q^*$ is a
global minimizer of the relaxed population risk.

Since $N>2m$, we have $m/N<1/2$. Thus $g_0$ is the unique modal gate
of every $q_{i}^*$, and componentwise argmax yields
$\vec\garg
    =
    (g_0,\ldots,g_0),
$ and $
    f_{\vec\garg}(X)=z_0
$.
On the other hand, any discrete configuration containing exactly $m$
copies of $g_1$ and $N-m$ copies of $g_0$ satisfies
$
    f_{\vec{g}}(X)=z^*
$.
It therefore attains the globally minimal population risk and is, in
particular, optimal among the discrete configurations. Hence, with $\risk(z):=\ex[(X,Y)\sim\setify{D}]{\ell(z,Y)}$, $ \rdisc^*
    =
    \risk(z^*)$
whereas
$
    \rdisc(\vec{g}^{\mathrm{arg}})
    =
    \risk(z_0)
$.
By~\cref{eq:risk-separation},
\begin{equation}
    \gapopt
    =
    \rdisc(\vec\garg)
    -
    \rdisc^*
    =
    \risk(z_0)-\risk(z^*)
    =
    c>0,
    \nonumber
\end{equation}
independently of $N$.
It remains to show local concentration. For every neuron,
\begin{equation}
    q_{i}^*
    =
    \left(1-\frac{m}{N}\right)\delta_{g_0}
    +
    \frac{m}{N}\delta_{g_1}
    \longrightarrow
    \delta_{g_0}
    \qquad\text{as }N\to\infty.
    \nonumber
\end{equation}
Let $D$ be any local dispersion measure as defined earlier.
Continuity of $D$ at $\delta_{g_0}$ therefore implies
$
    D(q_{i}^*)
    \longrightarrow
    D(\delta_{g_0})
    =
    0
$.
Since all $N$ trainable neurons have the same distribution,
\begin{equation}
    \max_{1\le i\le N}D(q_{i}^*)
    \longrightarrow 0.\nonumber
\end{equation}
Thus the local gate distributions become arbitrarily concentrated while
the argmax optimality gap remains equal to the positive constant $c$.
\end{proof}
\subsection{Examples for the Local-Concentration Counterexample}
\label{app:counterexamples}

We instantiate \cref{thm:concentration-insufficient} for two standard
supervised-learning losses. In both cases, the population risk admits a
finite global minimizer $z^*$ whose risk is strictly smaller than that of
an alternative prediction $z_0$. The construction in
\cref{thm:concentration-insufficient} then yields globally optimal relaxed
DLGNs whose local gate distributions converge uniformly to point masses,
while componentwise argmax retains a strictly positive discrete optimality
gap.

\paragraph{Squared loss.}
Consider a scalar prediction $z\in\mathbb{R}$ and squared loss
$\ell(z,y)=(z-y)^2$.
Let $Y$ be any random variable with finite second moment and mean $\ex[(X,Y)\sim\setify{D}]{Y}=z^*$.
The corresponding population risk is
\begin{align}
    \risk(z)
    =
    \ex[(X,Y)\sim\setify{D}]{(z-Y)^2}
    =
    (z-z^*)^2+\operatorname{Var}_{(X,Y)\sim\setify{D}}(Y),\nonumber
\end{align}
and is therefore globally minimized at $z=z^*$.
For any $z_0\neq z^*$,
\begin{equation}
    \ex[(X,Y)\sim\setify{D}]
    {\ell(z_0,Y)-\ell(z^*,Y)}
    =
    (z_0-z^*)^2
    =:c_{\mathrm{MSE}}>0.\nonumber
\end{equation}
Hence the assumptions of \cref{thm:concentration-insufficient} hold.
For example, choosing
$Y\sim\operatorname{Bernoulli}(0.64)$, $z^*=0.64$ and $z_0=0$
gives $c_{\mathrm{MSE}}
    =
    (0.64)^2
    =
    0.4096
$.
Thus the theorem constructs a sequence of globally optimal relaxed DLGNs
whose gate distributions become arbitrarily concentrated, while the
argmax discretization remains separated from the optimal discrete risk by
$0.4096$.

\paragraph{Binary cross-entropy.}
Consider binary classification with label $Y\sim\operatorname{Bernoulli}(\pi)$, $
    \pi\in(0,1)$
and let $z\in\mathbb{R}$ denote the class-$1$ logit relative to a fixed
class-$0$ logit of zero. The predicted class-$1$ probability is
$\sigma(z)$, where $\sigma$ denotes the logistic sigmoid. The binary
cross-entropy loss can be written as
    $\ell(z,y)
    =
    \log(1+e^z)-yz$.
Its population risk is therefore
\begin{equation}
   \risk(z)
    =\log(1+e^z)-\ex[(X,Y)\sim\setify{D}]
    {Y} z=
    \log(1+e^z)-\pi z.\nonumber
\end{equation}
Differentiating gives
\begin{equation}
    \risk'(z)
    =
    \sigma(z)-\pi,
    \qquad
    \risk''(z)
    =
    \sigma(z)(1-\sigma(z))>0.\nonumber
\end{equation}
Consequently, the unique global minimizer is
    $z^*
    =
    \sigma^{-1}(\pi)$.
Taking $z_0=0$ and any $\pi\neq 1/2$ yields
\begin{align}
    c_{\mathrm{CE}}
    :=
    \risk(0)-\risk(z^*)
    =
    \log 2
    -
    \log(1+e^{z^*})
    +
    \pi z^*
    >0.\nonumber
\end{align}
Thus binary cross-entropy also satisfies the assumptions of
\cref{thm:concentration-insufficient}.
Equivalently, choosing any fixed $m>0$ and setting
$
    \pi=\sigma(m)
$
gives $z^*=m$ and
\begin{equation}
    c_{\mathrm{CE}}
    =
    \log 2
    -
    \log(1+e^m)
    +
    m\sigma(m)
    >0.\nonumber
\end{equation}
Applying the construction of \cref{thm:concentration-insufficient} with
$z_0=0$ and $z^*=m$ therefore yields a sequence for which the relaxed
network remains globally optimal and every local gate
distribution converges to a point mass, while the argmax optimality gap
remains equal to the positive constant $c_{\mathrm{CE}}$.

\subsection{Proof of~\cref{eq:ce-risk-bound}}
\label{app:ce-proof}
\begin{proof}
Before any gate is fixed, the conditional expected risk equals $\ex[\vec{G}\sim q]
     {\rdisc(\vec{G})}$.
Suppose that $\vec g_{<i}$ has been fixed. By the law of total expectation,
\begin{equation}
    \ex[G_i\sim q_i]
    {
        C_i^{\mathrm{CE}}(G_i)
    }
    =
    \ex[\vec{G}_{\geq i}\sim q_{\geq i}]
    {
        \rdisc(\vec g_{<i},\vec{G}_{\geq i})
    }.\nonumber
\end{equation}
Since the minimum of a finite set cannot exceed its weighted average, $g_i^{\mathrm{CE}}$ satisfies
\begin{equation}
    \min_{g\in\mathcal G_i}C_i^{\mathrm{CE}}(g)=C_i^{\mathrm{CE}}(g_i^{\mathrm{CE}})
    \leq
    \ex[G_i\sim q_i]
    {
        C_i^{\mathrm{CE}}(G_i)
    }.\nonumber
\end{equation}
Thus fixing each gate according to $\argmin_{g}C_i^{\mathrm{CE}}(g)$ cannot increase
the conditional expected risk. Repeating this argument for
$i=1,\ldots,N$ gives
\begin{equation}
    \rdisc(\vec{g}^{\mathrm{CE}})\leq \ex[\vec{G}\sim q]
     {\rdisc(\vec{G})}.\nonumber
\end{equation}
The lower bound follows immediately from the definition of
$\rdisc^*$.
\end{proof}

\subsection{Proof of~\cref{thm:ce-relaxed}}
\begin{proof}
Introduce 
$
    C_i^\mu(k\mid g_{<i})
    :=
    \ex[(X,Y)\sim\setify{D}]
    {
        \ell(\mu_{i,k}(X),Y)
    }
$ as the risk of the mean completion prediction.
By the triangle inequality,
\begin{align}
    \left|
        C_i^{\mathrm{CE}}(k)
        -
        C_i^{\mathrm{rel}}(k)
    \right|
    \leq
    \left|
        C_i^{\mathrm{CE}}(k)
        -
        C_i^\mu(k)
    \right|+
    \left|
        C_i^\mu(k)
        -
        C_i^{\mathrm{rel}}(k)
    \right|.\nonumber
\end{align}
For the first term, Lipschitz continuity gives
\begin{align}
    \left|
        C_i^{\mathrm{CE}}(k)
        -
        C_i^\mu(k)
    \right|
    \leq
    L
    \ex[(X,Y)~\sim\setify{D},\vec{G}_{>i}\sim q_{>i}]
    {
        \norm{
            f_{g_{<i},k,\vec{G}_{>i}}(X)
            -
            \mu_{i,k}(X)
        }
    }
    \leq
    L V_{i,k},\nonumber
\end{align}
where the final inequality follows from Cauchy--Schwarz.
Similarly,
\begin{equation}
    \left|
        C_i^\mu(k)
        -
        C_i^{\mathrm{rel}}(k)
    \right|
    \leq
    L P_{i,k}.\nonumber
\end{equation}
Combining both bounds proves~\cref{eq:ce-relaxed-bound}.
\end{proof}

\subsection{Proof of~\cref{thm:first-order-selection}}
\label{app:first-order-proof}

\begin{proof}
For fixed $i$, $k$, and sample $(x,y)$, let
$
    \phi_i(a;x,y)
    :=
    \ell(F_{i,a}(x),y)
$,
where $F_{i,a}$ denotes the current hybrid network with the scalar
activation of neuron $i$ replaced by $a$, leaving the downstream
computation unchanged. Thus, the relaxed and intervened networks
correspond to $a_i(x)$ and $a_{i,k}(x)$, respectively.
Taylor's theorem gives
\begin{equation}
    \phi_i(a_{i,k};x,y)
    =
    \phi_i(a_i;x,y)
    +
    \phi_i'(a_i;x,y)\Delta a_{i,k}(x)
    +
    r_{i,k}(x,y).\nonumber
\end{equation}
By the curvature assumption,
\begin{equation}
    |r_{i,k}(x,y)|
    \leq
    \frac{M_i}{2}\Delta a_{i,k}(x)^2.\nonumber
\end{equation}
Moreover,
$
    \phi_i'(a_i;X,Y)
    =
    \frac{\partial \ell}{\partial a_i}
$
and taking expectations therefore yields
\begin{equation}
    \left|
        \crel{i}{k}
        -
        \left[
            \crel{i}{\mathrm{soft}}+S_i(k)
        \right]
    \right|
    \leq
    \frac{M_i}{2}
    \ex[(X,Y)\sim\setify{D}]
    {
        \Delta a_{i,k}(X)^2
    }.
    \label{eq:taylor-remainder}
\end{equation}
Combining~\cref{eq:taylor-remainder} with
\cref{thm:ce-relaxed} and the triangle inequality gives
\begin{align}
    \left|
        \cce{i}{k}
        -
        \left[
            \crel{i}{\mathrm{soft}}+S_i(k)
        \right]
    \right|
    &\leq
    \left|
        \cce{i}{k}-\crel{i}{k}
    \right|
    +
    \left|
        \crel{i}{k}
        -
        \left[
            \crel{i}{\mathrm{soft}}+S_i(k)
        \right]
    \right|
    \nonumber\\
    &\leq
    L_\ell(V_{i,k}+P_{i,k})
    +
    \frac{M_i}{2}
    \ex[(X,Y)\sim\setify{D}]
    {
        \Delta a_{i,k}(X)^2
    }
    =
    \epsilon_{i,k},\nonumber
\end{align}
which proves~\cref{eq:first-order-ce-bound}.
Now let
$
    k_i^{\mathrm{CE}}
    \in
    \argmin_{k\in\cG}\cce{i}{k}.
$
Since $k_i^{\mathrm{FO}}$ minimizes $S_i(k)$,
$
    S_i(k_i^{\mathrm{FO}})
    \leq
    S_i(k_i^{\mathrm{CE}}).
$
Applying~\cref{eq:first-order-ce-bound} to both candidates gives
\begin{align}
    \cce{i}{k_i^{\mathrm{FO}}}
    &\leq
    \crel{i}{\mathrm{soft}}
    +
    S_i(k_i^{\mathrm{FO}})
    +
    \epsilon_i
   \leq
    \crel{i}{\mathrm{soft}}
    +
    S_i(k_i^{\mathrm{CE}})
    +
    \epsilon_i
    \leq
    \cce{i}{k_i^{\mathrm{CE}}}
    +
    2\epsilon_i.\nonumber
\end{align}
Since
$\cce{i}{k_i^{\mathrm{CE}}}
=
\min_{k\in\cG}\cce{i}{k}$,
this proves~\cref{eq:first-order-ce-bound}.
For proving~\cref{eq:fo-beats-argmax}, let
$
    k_i^{\mathrm{arg}}
    \in
    \argmax_{k\in\cG} q_i(k).
$
Applying~\cref{eq:first-order-ce-bound} to
$k_i^{\mathrm{FO}}$ and $k_i^{\mathrm{arg}}$ gives
\begin{align}
    \cce{i}{k_i^{\mathrm{FO}}}
    -
    \cce{i}{k_i^{\mathrm{arg}}}
    \leq
    S_i(k_i^{\mathrm{FO}})
    -
    S_i(k_i^{\mathrm{arg}})
    +
    2\epsilon_i.\nonumber
\end{align}
Hence
\begin{equation}
    S_i(k_i^{\mathrm{arg}})
    -
    S_i(k_i^{\mathrm{FO}})
    >
    2\epsilon_i
    \quad\Longrightarrow\quad
    \cce{i}{k_i^{\mathrm{FO}}}
    <
    \cce{i}{k_i^{\mathrm{arg}}}.\nonumber
\end{equation}
\end{proof}

\section{Additional Experimental Details}
\label{app:experimental-details}

\subsection{Model Architectures}
\label{app:model-architectures}

We consider two convolutional DLGN families on CIFAR-10. Our
\emph{Compact-S/M} models follow the convolutional architecture of
\citet{wang2026learning}, while \emph{TreeLogicNet-S/M} follow the
TreeLogicNet architecture of \citet{petersen2024convolutional}.
The two families differ substantially in their convolutional
parameterization: Compact models use a single Boolean operation per
convolutional neuron together with connection learning, whereas
TreeLogicNet uses depth-$3$ Boolean trees containing seven trainable
gates per convolutional kernel.

\paragraph{Compact models.}
For CIFAR-10, the Compact architecture consists of four convolutional
layers followed by two fully connected logic layers and a GroupSum
classification head. The width parameter is $k=128$ for Compact-S and
$k=256$ for Compact-M. Following \citet{wang2026learning}, the input is
thermometer encoded using $N=3$ thresholds for S and $N=7$ thresholds
for M. Each convolutional kernel has a $3\times3$ receptive field.
The exact architectures used in our experiments are summarized in
\cref{tab:compact-architectures}.

\begin{table*}[t]
    \centering
    \caption{\textbf{Compact CIFAR-10 architectures.}
    Compact-S uses $k=128$ and $N=3$ input thresholds; Compact-M uses
    $k=256$ and $N=7$. Convolutional layers use $3\times3$ receptive
    fields and padding $1$.}
    \label{tab:compact-architectures}
    \scriptsize
    \setlength{\tabcolsep}{5pt}
    \begin{tabularx}{\textwidth}{
        l
        >{\centering\arraybackslash}X
        >{\centering\arraybackslash}X
        >{\centering\arraybackslash}X
        >{\centering\arraybackslash}X
    }
        \toprule
        \textbf{Stage}
        & \textbf{Operation}
        & \textbf{Stride}
        & \textbf{Compact-S ($k=128$)}
        & \textbf{Compact-M ($k=256$)} \\
        \midrule
        Input
        & Thermometer encoding
        & --
        & $(9,32,32)$
        & $(21,32,32)$ \\

        Conv 1
        & $3\times3$ logic conv.
        & $2$
        & $(128,16,16)$
        & $(256,16,16)$ \\

        Conv 2
        & $3\times3$ logic conv.
        & $1$
        & $(128,16,16)$
        & $(256,16,16)$ \\

        Conv 3
        & $3\times3$ logic conv.
        & $2$
        & $(512,8,8)$
        & $(1024,8,8)$ \\

        Conv 4
        & $3\times3$ logic conv.
        & $1$
        & $(512,8,8)$
        & $(1024,8,8)$ \\

        Flatten
        & --
        & --
        & $32{,}768$
        & $65{,}536$ \\

        Logic 1
        & Logic layer
        & --
        & $80{,}000$
        & $160{,}000$ \\

        Logic 2
        & Logic layer
        & --
        & $80{,}000$
        & $160{,}000$ \\

        Output
        & GroupSum
        & --
        & $10$
        & $10$ \\
        \bottomrule
    \end{tabularx}
\end{table*}

During training, Compact neurons jointly represent gate and connection
choices: the 16 candidate Boolean operations are associated with
independently sampled input pairs from the local receptive field.
Connection learning proceeds through adaptive resampling of these
candidates. After discretization, only one Boolean operation and its
corresponding input pair remain per neuron.

\paragraph{TreeLogicNet models.}
The TreeLogicNet models follow the CIFAR-10 architecture of
\citet{petersen2024convolutional}. Each convolutional block uses
depth-$3$ binary logic trees: eight sampled inputs are combined by seven
trainable Boolean gates into one output. Each block is followed by fixed
$2\times2$ OR pooling with stride $2$. Four such blocks reduce the
spatial resolution from $32\times32$ to $2\times2$, after which three
randomly connected logic layers and a GroupSum head perform
classification. The width parameter is $k=32$ for TreeLogicNet-S and
$k=256$ for TreeLogicNet-M. Both models use three thresholds for the
CIFAR-10 input encoding.

\begin{table*}[t]
    \centering
    \caption{\textbf{TreeLogicNet CIFAR-10 architectures.}
Each convolutional kernel is a depth-$3$ Boolean tree with seven
trainable gates, followed by fixed $2\times2$ OR pooling. S and M
correspond to $k=32$ and $k=256$, respectively.}
    \label{tab:treelogicnet-architectures}
    \scriptsize
    \setlength{\tabcolsep}{5pt}
    \begin{tabularx}{\textwidth}{
        l
        >{\centering\arraybackslash}X
        >{\centering\arraybackslash}X
        >{\centering\arraybackslash}X
    }
        \toprule
        \textbf{Stage}
        & \textbf{Operation}
        & \textbf{TreeLogicNet-S ($k=32$)}
        & \textbf{TreeLogicNet-M ($k=256$)} \\
        \midrule
        Input
        & Three-threshold encoding
        & $(9,32,32)$
        & $(9,32,32)$ \\

        Conv 1
        & $3\times3$, tree depth $3$
        & $32$ channels
        & $256$ channels \\

        Pool 1
        & $2\times2$ OR, stride $2$
        & $(32,16,16)$
        & $(256,16,16)$ \\

        Conv 2
        & $3\times3$, tree depth $3$
        & $128$ channels
        & $1024$ channels \\

        Pool 2
        & $2\times2$ OR, stride $2$
        & $(128,8,8)$
        & $(1024,8,8)$ \\

        Conv 3
        & $3\times3$, tree depth $3$
        & $512$ channels
        & $4096$ channels \\

        Pool 3
        & $2\times2$ OR, stride $2$
        & $(512,4,4)$
        & $(4096,4,4)$ \\

        Conv 4
        & $3\times3$, tree depth $3$
        & $1024$ channels
        & $8192$ channels \\

        Pool 4
        & $2\times2$ OR, stride $2$
        & $(1024,2,2)$
        & $(8192,2,2)$ \\

        Flatten
        & --
        & $4096$
        & $32{,}768$ \\

        Logic 1
        & Random logic layer
        & $40{,}960$
        & $327{,}680$ \\

        Logic 2
        & Random logic layer
        & $20{,}480$
        & $163{,}840$ \\

        Logic 3
        & Random logic layer
        & $10{,}240$
        & $81{,}920$ \\

        Output
        & GroupSum
        & $10$
        & $10$ \\
        \bottomrule
    \end{tabularx}
\end{table*}

The two architecture families therefore provide complementary testbeds.
Compact-S/M use shallow single-gate convolutional kernels with learned
connectivity, whereas TreeLogicNet-S/M retain deeper tree-structured
convolution and fixed random connectivity. This allows
us to test whether the observed discretization behavior persists across
substantially different convolutional DLGN parameterizations.

\subsection{Training and Freezing Hyperparameters}
\label{app:training-details}

All models are trained on the 45k/5k training--validation split of
CIFAR-10 described in \cref{sec:experiments}, using cross-entropy loss
and GroupSum aggregation. Compact models use Adam and are trained for
600k steps, while TreeLogicNet models use AdamW with weight decay
$0.002$ and are trained for 200k steps. Unless stated otherwise, all
reported S-model experiments use five independent seeds.

For progressive discretization, freezing begins only after an initial
relaxed training period. TaskFreeze evaluates the first-order argmax
freeze shock every $K$ training steps, maintains an EMA with factor
$\beta$, and freezes a neuron once its smoothed score remains negative
for $P$ consecutive evaluations. EntropyFreeze uses the same
neuron-wise progressive framework but determines eligibility from
convergence of the local gate entropy. LayerEntropyFreeze follows the
shallow-to-deep layer-wise entropy criterion of
\citet{wang2026learning}, while GlobalFreeze simultaneously fixes all
neurons to their current argmax gates at the discretization start.
FinalArgmax keeps all gates relaxed throughout training and applies
argmax only for discrete evaluation.

The exact optimization and freezing hyperparameters used in our
experiments are summarized in \cref{tab:training-hyperparameters}.

\begin{table}[t]
    \centering
    \caption{\textbf{Training and freezing hyperparameters.}
    $H$ denotes entropy convergence and $S$ first-order freeze shock.
    Values are shared across the Compact and TreeLogicNet experiments
    unless shown otherwise.}
    \label{tab:training-hyperparameters}
    \scriptsize
    \setlength{\tabcolsep}{4pt}
    \begin{tabularx}{\linewidth}{
        >{\raggedright\arraybackslash}X
        >{\centering\arraybackslash}X
        >{\centering\arraybackslash}X
    }
        \toprule
        \textbf{Parameter} & \textbf{Compact} & \textbf{TreeLogicNet} \\
        \midrule
        Optimizer & Adam & AdamW \\
        Weight decay & $0$ & $0.002$ \\
        Training steps & $600$k & $200$k \\
        Batch size & $128$ & $128$ \\
        Learning rate & $0.02$ & $0.02$ \\
        LR schedule & Constant & Constant \\
        Scoring interval $K$ & $K_H=1$, $K_S=1000$
                             & $K_H=1$, $K_S=1000$ \\
        EMA factor $\beta$ & $\beta_H=0.99$, $\beta_S=2^{-K_S/10\,000}$
                           & $\beta_H=0.99$, $\beta_S=2^{-K_S/10\,000}$ \\
        Patience $P$ & $P_H=200$, $P_S=3$
                     & $P_H=200$, $P_S=3$ \\
        Entropy criterion & $|H_t-\bar H_{t-1}|<5\times10^{-4}$
                          & $|H_t-\bar H_{t-1}|<5\times10^{-4}$ \\
        \bottomrule
    \end{tabularx}
\end{table}

\subsection{Diagnostic Protocol}
\label{app:diagnostic-protocol}

Our diagnostic experiments measure the effect of discretizing individual
neurons while keeping the remainder of the network relaxed. For neuron
$i$ and candidate gate $k$, we define the realized intervention effect
as
\[
    \Delta C_i^{\mathrm{rel}}(k)
    :=
    C_i^{\mathrm{rel}}(k)
    -
    C_i^{\mathrm{rel}}(\mathrm{soft}),
\]
where only neuron $i$ is replaced by the discrete gate $k$. Negative
values therefore indicate that the intervention decreases the task loss.

For the entropy and argmax-score diagnostics, the intervention gate is
$k_i^{\mathrm{arg}}=\argmax_k q_i(k)$. For score-based selection, we
instead use $k_i^*=\argmin_k S_i(k)$. The corresponding first-order
scores are computed from the same relaxed network state without
additional downstream forward passes. Unless explicitly stated
otherwise, diagnostic loss changes and scores are evaluated on the
validation set.

To test approximation locality, we additionally measure the squared
activation displacement
\[
    D_i(k)
    :=
    \mathbb{E}_{(X,Y)\sim\mathcal D}
    [\Delta a_{i,k}(X)^2],
    \qquad
    \Delta a_{i,k}:=a_{i,k}-a_i.
\]
For each neuron, the all-candidate analysis evaluates all 16 Boolean
gates and relates $D_i(k)$ to the error of the first- and second-order
loss predictions. The second-order score is
\[
    S_i^{(2)}(k)
    :=
    S_i(k)
    +
    \frac{1}{2}
    \Delta a_{i,k}^{\top}H_i\Delta a_{i,k},
\]
where the quadratic form is evaluated using Hessian--vector products
without explicitly constructing the Hessian. We report Spearman
correlations across candidates within each neuron and aggregate these
statistics across neurons.

For the held-out selection experiment in \cref{fig:locality}c, the
validation data are divided into two disjoint halves. The first half is
used to compute first-order task-aware choices and activation
displacements, while realized loss effects are evaluated exclusively on
the second half. This separates gate selection from evaluation and
avoids measuring the apparent advantage of a selected intervention on
the same examples used to select it.

\subsection{Diagnostics for TreeLogicNet}
\label{app:treelogicnet-diagnostics}

We repeat the main diagnostic analyses on TreeLogicNet-S to assess
whether the observed behavior is specific to the Compact architecture.
\Cref{fig:correlation-hexbin-treelogicnet} shows the corresponding
freeze-effect analysis at the 25k discretization start. As for Compact-S,
gate entropy provides only a weak predictor of the realized argmax freeze
effect, while the first-order argmax freeze shock is consistently
positively associated with the measured loss change. Directly optimizing
the first-order score over candidate gates is substantially less
reliable and exhibits pronounced layer dependence. This confirms that
the distinction between using first-order information for
\emph{selection} and for \emph{assessment} is not specific to the
Compact architecture.

\begin{figure*}[t]
    \centering
    \includegraphics[width=1.0\textwidth]{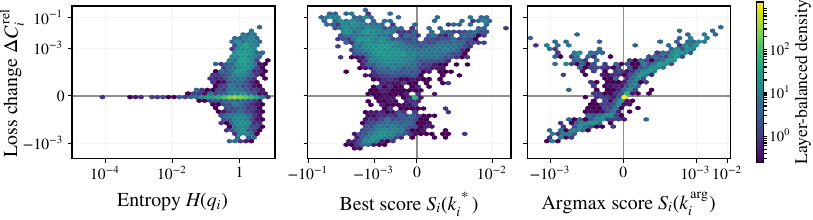}
    \vspace{-0.5cm}
    \caption{\textbf{Predicting freeze effects on TreeLogicNet-S.}
    Measured loss change from freezing to the argmax gate versus gate
    entropy (left), from freezing to the task-aware gate
    $k_i^{*}=\argmin_k S_i(k)$ versus its first-order score
    (center), and from freezing to the argmax gate versus its first-order
    score (right). As for Compact-S, score-based selection is less
    reliably aligned with the realized effect, whereas the argmax freeze
    shock is strongly predictive.}
    \label{fig:correlation-hexbin-treelogicnet}
\end{figure*}
We additionally compare the resulting progressive freezing dynamics in
\cref{fig:freezing-dynamics-treelogicnet}. The qualitative schedules
differ from Compact-S: EntropyFreeze commits a large fraction of neurons
soon after discretization begins and ultimately freezes more neurons
than TaskFreeze. Nevertheless, TaskFreeze again produces a gradual
neuron-wise transition rather than the abrupt commitment induced by
GlobalFreeze or LayerFreeze. The layer-wise view further shows that the
two neuron-wise criteria are not interchangeable: EntropyFreeze becomes
increasingly aggressive toward deeper convolutional layers, whereas
TaskFreeze distributes its decisions more evenly across depth. Thus,
the precise freezing trajectory is architecture dependent, while
task-based and entropy-based eligibility remain qualitatively distinct.
\begin{figure*}[t]
    \centering
    \begin{subfigure}{0.45\textwidth}
        \centering
        \includegraphics[width=1.0\textwidth]{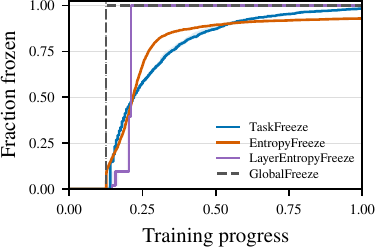}
        \caption{Global freezing dynamics.}
        \label{fig:freezing-global-treelogicnet}
    \end{subfigure}
    \hfill
    \begin{subfigure}{0.45\textwidth}
        \centering
        \includegraphics[width=1.0\textwidth]{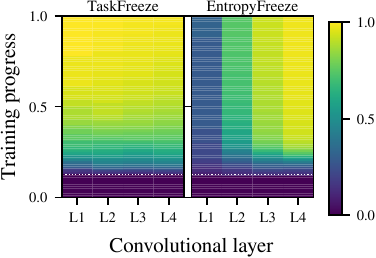}
        \caption{Layer-wise freezing dynamics.}
        \label{fig:freezing-layerwise-treelogicnet}
    \end{subfigure}

    \caption{\textbf{Progressive freezing dynamics on TreeLogicNet-S.}
    (a) Fraction of frozen neurons throughout training for different
    discretization criteria. (b) Layer-wise frozen fractions show that
    TaskFreeze progressively commits individual decisions, whereas
    EntropyFreeze induces a markedly different layer-dependent
    schedule.}
    \label{fig:freezing-dynamics-treelogicnet}
\end{figure*}

\subsection{Sensitivity to Discretization Start}
\label{app:discretization-start}

We additionally evaluate the sensitivity of the freezing methods to the
point at which discretization begins. Tables~\ref{tab:compact-s-start}
and~\ref{tab:treelogicnet-s-start} report final discrete CIFAR-10 test
accuracy across multiple start points for Compact-S and TreeLogicNet-S,
respectively. Overall, the relative performance of the methods is
reasonably stable across the considered schedules.
\begin{table*}[h]
    \centering
    \caption{\textbf{Sensitivity to discretization start on Compact-S.}
    Final discrete CIFAR-10 test accuracy (\%) for discretization starting
    at 200k, 300k, or 400k training steps, with connection resampling
    either continued or stopped after discretization begins. Values are
    mean $\pm$ sample standard deviation across available seeds.}
    \label{tab:compact-s-start}
    \scriptsize
    \setlength{\tabcolsep}{-3pt}

    \begin{tabularx}{\textwidth}{
        >{\raggedright\arraybackslash}X
        *{6}{>{\centering\arraybackslash}X}
    }
        \toprule
        & \multicolumn{3}{c}{\textbf{Continued resampling}}
        & \multicolumn{3}{c}{\textbf{Stopped resampling}} \\
        \cmidrule(lr){2-4}
        \cmidrule(lr){5-7}
        \textbf{Method}
        & \textbf{200k}
        & \textbf{300k}
        & \textbf{400k}
        & \textbf{200k}
        & \textbf{300k}
        & \textbf{400k} \\
        \midrule

        TaskFreeze (ours)
        & $66.13 \pm 0.20$
        & $66.52 \pm 0.17$
        & $66.25 \pm 0.27$
        & $65.55 \pm 0.29$
        & $65.77 \pm 0.15$
        & $65.93 \pm 0.15$ \\

        EntropyFreeze (ours)
        & $64.30 \pm 0.83$
        & $64.32 \pm 0.72$
        & $64.11 \pm 0.53$
        & $64.28 \pm 0.48$
        & $64.00 \pm 0.58$
        & $63.88 \pm 0.76$ \\

        GlobalFreeze (ours)
        & $66.27 \pm 0.51$
        & $66.17 \pm 0.10$
        & $66.41 \pm 0.21$
        & $65.84 \pm 0.15$
        & $66.04 \pm 0.26$
        & $66.12 \pm 0.37$ \\

        \midrule

        LayerEntropyFreeze
        & $64.75 \pm 0.80$
        & $65.24 \pm 0.53$
        & $64.10 \pm 0.45$
        & $65.94 \pm 0.14$
        & $65.95 \pm 0.26$
        & $65.99 \pm 0.17$ \\

        FinalArgmax
        & $64.63 \pm 0.61$
        & $64.63 \pm 0.61$
        & $64.63 \pm 0.61$
        & $64.63 \pm 0.61$
        & $64.63 \pm 0.61$
        & $64.63 \pm 0.61$ \\

        CAGE
        & $63.54 \pm 0.65$
        & $63.42 \pm 0.88$
        & $63.32 \pm 0.63$
        & $63.62 \pm 0.39$
        & $64.15 \pm 0.18$
        & $64.11 \pm 0.39$ \\

        \bottomrule
    \end{tabularx}
\end{table*}
Across both architectures, the main conclusions are robust to the
discretization start. TaskFreeze remains competitive across all tested
checkpoints, while no single freezing schedule consistently dominates.
For Compact-S, this behavior also persists when connection resampling is
stopped after discretization begins. Together with the layer-,
checkpoint-, and seed-wise diagnostics in the main text, these results
indicate that the observed distinction between task-based assessment and
entropy-based freezing is not tied to a particular training checkpoint
or connection-learning schedule.
\begin{table*}[t]
    \centering
    \caption{\textbf{Sensitivity to discretization start on
    TreeLogicNet-S.}
    Final discrete CIFAR-10 test accuracy (\%) for discretization
    starting at 10k, 25k, or 50k training steps. Values are mean $\pm$
    sample standard deviation across five seeds.}
    \label{tab:treelogicnet-s-start}
    \scriptsize
    \setlength{\tabcolsep}{5pt}

    \begin{tabularx}{\textwidth}{
        >{\raggedright\arraybackslash}X
        *{3}{>{\centering\arraybackslash}X}
    }
        \toprule
        \textbf{Method}
        & \textbf{10k}
        & \textbf{25k}
        & \textbf{50k} \\
        \midrule

        TaskFreeze (ours)
        & $60.42 \pm 0.63$
        & $60.76 \pm 0.46$
        & $60.26 \pm 0.80$ \\

        EntropyFreeze (ours)
        & $59.89 \pm 0.45$
        & $59.97 \pm 0.84$
        & $60.13 \pm 0.57$ \\

        GlobalFreeze (ours)
        & $60.05 \pm 0.82$
        & $60.50 \pm 0.52$
        & $60.78 \pm 0.28$ \\

        \midrule

        LayerEntropyFreeze
        & $60.72 \pm 0.57$
        & $60.63 \pm 0.55$
        & $60.81 \pm 0.42$ \\

        FinalArgmax
        & $59.29 \pm 0.78$
        & $59.29 \pm 0.78$
        & $59.29 \pm 0.78$ \\

        CAGE
        & $60.13 \pm 0.49$
        & $59.98 \pm 0.35$
        & $59.90 \pm 0.71$ \\

        \bottomrule
    \end{tabularx}
\end{table*}
\end{document}